\documentclass[conference]{IEEEtran}
\IEEEoverridecommandlockouts

\usepackage{cite}
\usepackage{amsmath,amssymb,amsfonts}
\usepackage{graphicx}
\usepackage{textcomp}
\usepackage{xcolor}

\usepackage{hyperref}
\usepackage{times}
\usepackage{booktabs}       
\usepackage{soul}
\usepackage{url}
\usepackage{graphicx}
\usepackage{amsmath}
\usepackage{amsthm}
\usepackage{algorithm}
\usepackage{algpseudocode}
\usepackage{wrapfig}

\newtheorem{theorem}{Theorem}
\newtheorem{definition}{Definition}

\def\BibTeX{{\rm B\kern-.05em{\sc i\kern-.025em b}\kern-.08em
    T\kern-.1667em\lower.7ex\hbox{E}\kern-.125emX}}
\begin{document}

\title{Efficient Manifold-Constrained Neural ODE for High-Dimensional Datasets\\
}

\author{\IEEEauthorblockN{1\textsuperscript{st} Muhao Guo}
\IEEEauthorblockA{\textit{Department of ECEE} \\
\textit{Arizona State University}\\
Tempe, United States\\
mguo26@asu.edu}
\and
\IEEEauthorblockN{2\textsuperscript{nd} Haoran Li}
\IEEEauthorblockA{\textit{Department of ECEE} \\
\textit{Arizona State University}\\
Tempe, United States \\
lhaoran@asu.edu}
\and
\IEEEauthorblockN{3\textsuperscript{rd} Yang Weng*}
\IEEEauthorblockA{\textit{Department of ECEE} \\
\textit{Arizona State University}\\
Tempe, United States \\
yang.weng@asu.edu}
}

\maketitle

\begin{abstract}
Neural ordinary differential equations (NODE) have garnered significant attention for their design of continuous-depth neural networks and the ability to learn data/feature dynamics. However, for high-dimensional systems, estimating dynamics requires extensive calculations and suffers from high truncation errors for the ODE solvers. To address the issue, one intuitive approach is to consider the non-trivial topological space of the data distribution, i.e., a low-dimensional manifold. Existing methods often rely on knowledge of the manifold for projection or implicit transformation, restricting the ODE solutions on the manifold. Nevertheless, such knowledge is usually unknown in realistic scenarios. Therefore, we propose a novel approach to explore the underlying manifold to restrict the ODE process. Specifically, we employ a structure-preserved encoder to process data and find the underlying graph to approximate the manifold. Moreover, we propose novel methods to combine the NODE learning with the manifold, resulting in significant gains in computational speed and accuracy. Our experimental evaluations encompass multiple datasets, where we compare the accuracy, number of function evaluations (NFEs), and convergence speed of our model against existing baselines. Our results demonstrate superior performance, underscoring the effectiveness of our approach in addressing the challenges of high-dimensional datasets.
\end{abstract}

\section{Introduction}
Understanding and modeling the dynamics of complex systems is a fundamental challenge in various fields, 
including physics~\cite{weng2022transform, cui2023sig2vec}, biology~\cite{guo2023identifying}, engineering \cite{guo2023advantage}, natural language processing~\cite{guo2024transparent, guo2023msq, guo2024bayesian}, and large language models~\cite{guo2024bias}.
To learn the dynamics, two basic components need to be considered. The first is to learn a latent representation of the state of the system, and another is to learn how the latent state representation evolves forward in time. 

NODEs~\cite{chen2018neural} have emerged as a powerful framework for learning dynamics from data efficiently.
Their continuous-time modeling capabilities make them particularly suited for interpreting how the latent state representation evolves over time.
The core idea of NODEs is to use a neural network to parameterize a vector field~\cite{chen2018neural}, which is typically represented by a simple neural network~\cite{kidger2022neural}. 
The neural network considers the current state of the system as input and produces the time derivative of that state as output, which determines how the system will change over time. By integrating the vector field over time, it is possible to calculate the system's trajectory and make predictions about its future behavior.

However, capturing accurate representations in high-dimensional spaces with complex, unknown dynamics remains a significant challenge \cite{guo2023graph, guo2022patients}.
Existing methods often depend on numerical integration techniques~\cite{pal2021opening} 
or extend into higher-dimensional spaces~\cite{dupont2019augmented}. These approaches, while useful, can lead to increased computational complexity or introduce biases in the modeling process.
Recent work~\cite{lou2020neural} proposes to implicitly parameterize the original space with fewer parameters in the manifold. 
This achieves the state-of-the-art for system dynamics on a manifold, which is often the case for diversified engineering systems due to system constraints and conservation laws. This is because the proper introduction of the manifold enables the minimal truncation error for an ODE solver, much better than calculating ODEs in the ambient space. However, it requires the knowledge of the manifold.

\begin{figure}[t]
\centerline{
\includegraphics[trim=0cm 0cm 0cm 0cm, clip, width=1\columnwidth]{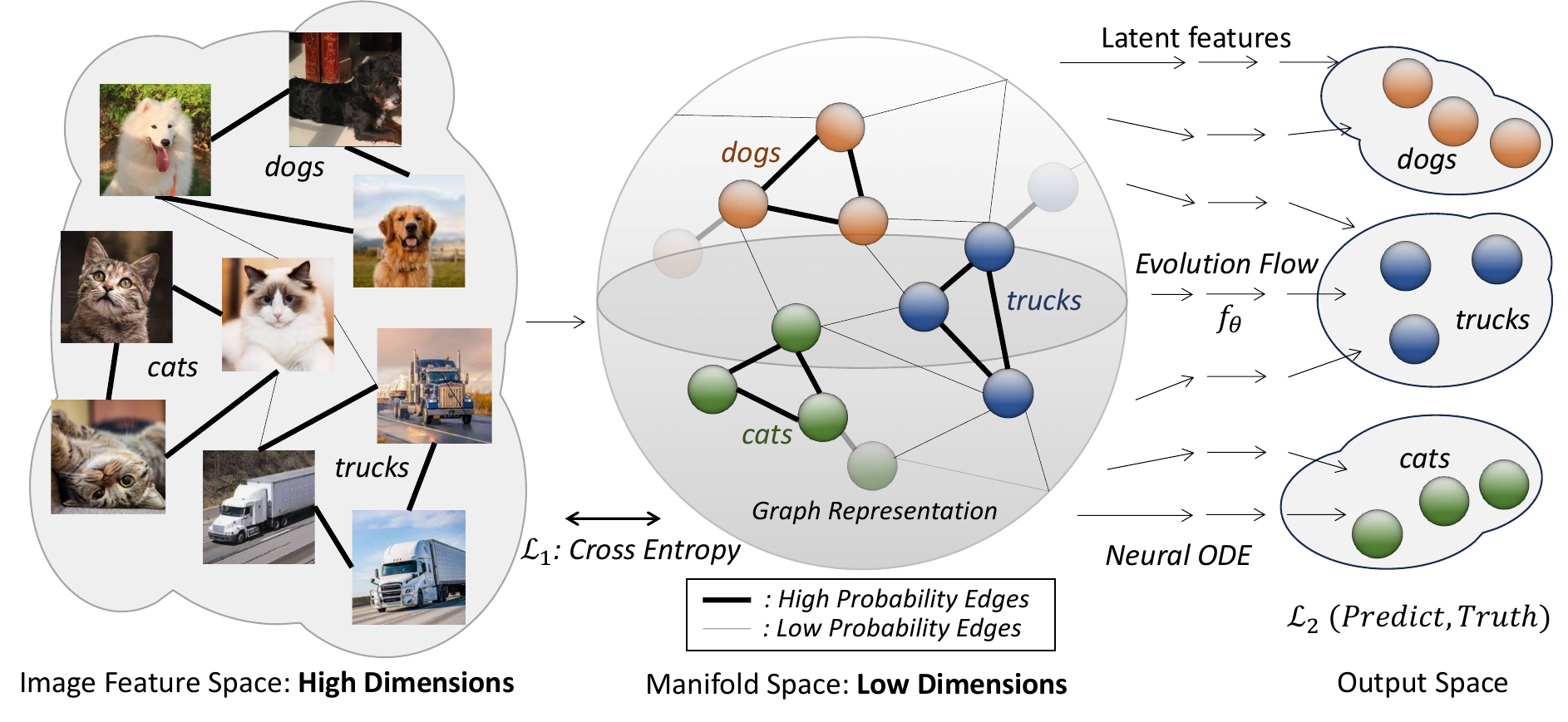}
}
\caption{Illustration of Manifold NODE Methodology. Original high-dimensional data is projected into a lower-dimensional manifold space, preserving intrinsic data structures. This projection is achieved by minimizing the cross-entropy loss between edge probability distributions of the original and manifold spaces. In this compact manifold space, NODEs are utilized to efficiently and accurately learn the evolution of latent features. The transition of learned flow from the latent space to the output space enhances the interpretability and transparency of the learning process.}
\label{fig: framework}
\end{figure}

One promising avenue to address these challenges lies in the field of manifold learning~\cite{lou2020neural,floryan2022data,lin2008riemannian}, a powerful approach that enables us to capture and represent the underlying structure of high-dimensional data. 
They generally start from a key assumption sometimes called the manifold hypothesis \cite{fefferman2016testing}, namely, the data lie on or near a low-dimensional manifold in state space.
Manifold learning techniques aim to uncover the intrinsic low-dimensional manifolds within complex, high-dimensional datasets. By doing so, they provide valuable insights into the underlying dynamics of systems. 

In this work, we propose an innovative data-driven approach tailored to learning dynamics.
By harnessing the principles of manifold learning, we focus on how the latent features evolve in the most representative manifold space.
This strategy not only simplifies the complexity inherent in dynamic learning but also ensures the preservation of accuracy, which is crucial when handling real-world data.
To guarantee that dynamic estimation can be conducted on the low dimensional feature space, the key is to preserve the manifold metric and structure in the feature space. This is because it ensures that the intrinsic geometric and topological relationships of the original high-dimensional data are maintained. This fidelity is essential for accurately modeling and predicting the system's dynamics, as it retains the fundamental characteristics and behaviors present in the original data space.

To do so, we first focus on how to accurately represent the structure of the original data. Graphs are highly effective in depicting both local and global structures within data, making graph construction an essential aspect of structure representation, particularly in the context of unknown or complex data. 
Equally important is the preservation of the data's inherent structure when transitioning to a low-dimensional space. The fundamental principle here is that data points close to each other in the original high-dimensional space should maintain their proximity in the reduced manifold space. This preservation ensures that the intrinsic relationships and structures within the data are not lost during the dimensionality reduction process.

Driven by these considerations, our approach begins with constructing a graph between datasets, followed by optimizing an embedding in a low-dimensional space that maintains the graph's structure. These concepts have been widely employed in various manifold learning methods, such as t-distributed Stochastic Neighbor Embedding (t-SNE) \cite{van2008visualizing} and Uniform Manifold Approximation and Projection (UMAP) \cite{sainburg2021parametric}.

In our methodology, we employ an encoder to derive the latent state representation in manifold space from the original space. This encoder is specifically trained using a cross-entropy cost function, which aims to minimize the discrepancy between a probability distribution in the manifold space and the distribution in the input space. These distributions are derived from the probabilistically weighted edges in the graph structures within manifold space and original space.
Concurrently, the manifold space is learned while a NODE determines the evolution of the latent state over time within this space. The ultimate objective is to map the latent state in the manifold space to a label space, as is typical in classification tasks. This specialized design of the encoder acts as a regularization mechanism for the temporal evolution of the latent state. 
It allows for parametric adjustments in preserving the global structure while enhancing the accuracy of the classifier by effectively capturing the inherent structure of the data. Our methodology, including these components and their interactions, is illustrated in Figure \ref{fig: framework}.

Our methodology addresses the complexities of dynamic learning in high-dimensional spaces by effectively reducing data to their intrinsic dimensionality within a nonlinear manifold. Given that image datasets typically exist in high-dimensional spaces, we have validated our approach across several image datasets, in addition to various time-series datasets. The results demonstrate that our model not only achieves superior accuracy but also operates at a faster speed and with fewer NFEs compared to baseline models.

\section{Related Work}
\subsection{Neural ODE (NODE)}
The basic idea of neural ordinary differential equations was originally considered in ~\cite{rico1992discrete}.
After ~\cite{chen2018neural} specified the architecture of NODEs and led to an explosion of applications in dynamic learning. For example, image classification~\cite{dupont2019augmented}, 
time series prediction~\cite{guo2023continuous},
time series classification ~\cite{kidger2020neural}, and continuous normalizing flows ~\cite{du2022flow}. According to ~\cite{chen2018neural}, the scalar-valued loss with respect to all inputs of any ODE solver can be computed directly without backpropagating through the operations of the solver. The intermediate quantities of the forward pass will not need to be stored. It causes the NODEs can be trained with a constant memory cost.

\subsection{Efficiency of NODEs}
As a continuous infinite-depth architecture, NODEs will bring several drawbacks.
The obvious drawback is that NODEs have a low training efficiency ~\cite{lehtimaki2022accelerating}.
To accelerate the training speed, several works have been done. Some works try to improve the efficiency of ODE solvers, such as regularizing the solver ~\cite{pal2021opening}, using interpolation backward dynamic methods ~\cite{daulbaev2020interpolation}, or using a second-order ODE optimizer ~\cite{liu2021second}.
Some works aim to optimize the objective function~\cite{xia2021heavy}. 
Simpler dynamics can lead to faster convergence and fewer discretizations of the solver ~\cite{finlay2020train}. 

One way to optimize the objective function is to take approximations of the learned dynamics. For example, ~\cite{finlay2020train} demonstrated that appropriate regularization of the learned dynamics can significantly accelerate training time without degrading performance. However, these approaches may be less accurate when encountering unsmooth dynamics, such as those with more oscillations or abrupt changes.
Other works are dedicated to optimizing the model structures, such as compressing the model. For example, ~\cite{lehtimaki2022accelerating} used model order reduction to obtain a smaller-size NODE model with fewer parameters. However, optimizing the model structure by compressing the model without considering the characteristics of the data can result in poor generalization capabilities.

\subsection{NODEs on Manifolds}
Manifold learning is a subfield of machine learning and dimensionality reduction that focuses on discovering the underlying structure or geometry of high-dimensional data. The central idea behind manifold learning is that many real-world datasets lie on or near lower-dimensional manifolds within the high-dimensional space \cite{floryan2022data}.
~\cite{lin2008riemannian} formulates the dimensionality reduction problem as a classical problem in Riemannian geometry. 
For dynamic learning, ~\cite{hairer2011solving} describes the differential equation on the manifold. Its solution evolves on a manifold, and the vector field is often only defined on this manifold. ~\cite{floryan2022data} explores the dynamics learning in the manifold using an auto-encoder. Our work utilizes the NODEs to learn better continuous dynamics.

Other works~\cite{mathieu2020riemannian}
also investigate the manifold generalization of NODEs. 
These works calculate either the change in probability with a Riemannian change of variables, or the change through the use of charts and Euclidean change of variables.
However, they are designed for normalizing flows, but the classification or regression task still remains to be investigated.

\section{Background and Preliminaries}
\subsection{Manifolds}
\noindent \textbf{Topological manifolds.}
A topological space $\mathcal{M}$ is a topological manifold of dimension $d$ if it satisfies the following conditions:
It is a second-countable Hausdorff space, ensuring that points can be separated by neighborhoods and that the topological structure is not too large.
It is locally Euclidean of dimension $d$, meaning that at every point on the manifold, there exists a small neighborhood where the space behaves like Euclidean space.
Furthermore, Whitney's embedding theorem~\cite{whitney1936differentiable} states that any $d$-dimensional manifold $\mathcal{M}^d$ can be embedded in $\mathbb{R}^{2d+1}$. This means that a space of at most $2d + 1$ dimensions is sufficient to represent a $d$-dimensional manifold.




\noindent \textbf{Differentiable manifolds.}
\label{Differentiable manifolds}
A topological manifold $\mathcal{M}$ is referred to as a smooth or differentiable manifold if it has the property of being continuously differentiable to any order. This implies that smooth functions can be defined on the manifold, making it suitable for calculus operations.

\begin{definition}[Smooth mapping]
\label{definition: Smooth mapping}
Consider two open sets, $U \subset \mathbb{R}^r$ and $V \subset \mathbb{R}^s$, and let $\mathcal{G} : U \rightarrow V$ be a function such that for $x \in U$ and $y \in V$, $\mathcal{G}(x) = y$.
If the function $\mathcal{G}$ has finite first-order partial derivatives, $ \frac{\partial{y_j}}{\partial{x_i}}$, for all $i = 1, 2, \cdots, r$, and all $j=1, 2, \dots, s$, then $\mathcal{G}$ is said to be a smooth (or differentiable) mapping on $U$. We also say that $\mathcal{G}$ is a $\mathcal{C}^1$-function on $U$ if all the first-order partial derivatives are continuous.
More generally, if $\mathcal{G}$ has continuous higher-order partial derivatives, $\frac{\partial^{k_1+ \cdots +k_r}{y_j}}{\partial{x_1^{k_1}} \cdots \partial{x_r^{k_r}}}$, for all $j = 1, 2, \cdots, s$ and all non-negative integers $k_1, k_2, \cdots, k_r$ such that $k_1 + k_2 + \cdots +k_r \leq r$, then we say that $\mathcal{G}$ is a $\mathcal{C}^r$-function, where $r = 1, 2, \cdots$.

\end{definition}
\begin{definition}[Diffeomorphism]
\label{definition: Diffeomorphism}
If $\mathcal{G}$ is a homeomorphism from an open set $U$ to an open set $V$, then $\mathcal{G}$ is said to be a $\mathcal{C}^r$ diffeomorphism if both $\mathcal{G}$ and its inverse $\mathcal{G}^{-1}$ are $\mathcal{C}^r$-functions.
\end{definition}

\begin{definition}[Diffeomorphic]
\label{definition: Diffeomorphic} $U$ and $V$ are diffeomorphic if there exists a diffeomorphism between them. 
\end{definition}

Following the Definition \ref{definition: Smooth mapping}, \ref{definition: Diffeomorphism}, and \ref{definition: Diffeomorphic}, we can straightforwardly extend these concepts to define diffeomorphism and diffeomorphic in manifolds~\cite{ma2011manifold}.

\begin{definition}[Diffeomorphism in manifolds]
\label{definition: Diffeomorphism in manifolds}
If $X$ and $Y$ are both smooth manifolds, a function $\mathcal{G} : X \rightarrow Y$ is a diffeomorphism if it is a homeomorphism from $X$ to $Y$ and both $\mathcal{G}$ and $\mathcal{G}^{-1}$ are smooth.
\end{definition}

\begin{definition}[Diffeomorphic of manifolds]
\label{definition: Diffeomorphic in manifolds}
Smooth manifolds $X$ and $Y$ are diffeomorphic if there exists a diffeomorphism between them. 
In this case, $X$ and $Y$ are essentially indistinguishable from each other. 
\end{definition}

\subsection{NODEs}
NODEs are a family of deep neural network models that can be interpreted as a continuous version of Residual Networks ~\cite{he2016identity}. Recall the formulation of a residual network:
\begin{equation}
    h_{t+1} - h_{t} = f(h_t, \theta_f),
    \label{res_net}
\end{equation}
where the $f$ is the residual block and the $\theta_f$ represents the parameters of $f$.
The left side of Equation \ref{res_net} can be seen as a denominator of $1$, so it can be represented by $\frac{h_{t+1} - h_{t}}{1} = f(h_t, \theta_f)$.
When the number of layers becomes infinitely large and the step becomes infinitely small, Equation \ref{res_net} will become an ODE format as shown in Equation \ref{ODE}.
\begin{equation}
    \lim_{dt \rightarrow 0} \frac{h_{t+dt} - h_{t}}{dt} = \frac{dh(t)}{dt} = f(h(t), t, \theta_f).
    \label{ODE}
\end{equation}
Thus, the NODE will have the same format as an ODE:
$h'(t) = f(h(t), t, \theta_f)$ and $h(0) = x_0$,
where $x_0$ is the input data.
Typically, $f$ will be some standard simple neural architecture, such as an MLP. The $\theta_f$ represents trainable parameters in $f$.
To obtain any final state of $h(t)$ when $t = T$, all that is needed is to solve an ODE with initial values, which is called an initial value problem (IVP):
\begin{equation}
    h(T) = h(0) + \int_{0}^{T} f(h(t), t, \theta_f)dt.
    \label{integrate}
\end{equation}
Thus, a NODE can transform from $h(0)$ to $h(T)$ through the solutions to the initial value problem (IVP) of the ODE.
This framework indirectly realizes a functional relationship $x \rightarrow F(x)$ as a general neural network.

By the properties of ODEs, NODEs are always invertible; we can reverse the limits of integration, or alternatively, integrate $-f$. The \textit{Adjoint Sensitivity Method}  ~\cite{pontryagin1961mathematical} based on reverse-time integration of an expanded ODE, allows for finding gradients of the initial value problem solutions $h(T)$ with respect to parameters $\theta_f$ and the initial values $h(0)$. 
This allows the training NODE to use gradient descent, which allows them to combine with other neural network blocks.

\section{Manifold-Constrained NODE}
\label{Learning Dynamics on Manifolds with Neural ODEs}
While NODEs excel at learning continuous transformations in Euclidean space, real-world data often exhibit complex geometries better represented by manifolds. By constraining NODEs to operate on a learned manifold, we can align the dynamics with the data's underlying topology, enhancing both interpretability and generalization.

Our approach begins with representing the structure in the manifold. Consider each sample as a data point within a given dataset. Assuming that data are uniformly distributed across a manifold in a warped data space, we calculate the distance between a data point and its $k^{th}$ nearest neighbor. This leads to the formulation of a likelihood scaled by these distances:$p_{j|i} = e^{\frac{-d(x_i, x_j) - \rho_i}{\sigma_i}}$,
where \( d(x_i, x_j) \) is the distance between points \( x_i \) and \( x_j \), \( \rho_i \) is a local connectivity parameter set to the distance from \( x_i \) to its nearest neighbor, and \( \sigma_i \) (default value is $log_2 k$) is a scaling parameter. This likelihood represents the probability of a point selecting another point as its neighbor. The global probability~\cite{sainburg2021parametric} is then defined as the probability of either of the two local probabilities occurring:
\begin{equation}
    p_{ij} = p_{j|i} + p_{i|j} - p_{j|i}p_{i|j}.
    \label{global_probability}
\end{equation}

Let $\mathcal{G}: \mathbb{R}^n \rightarrow \mathbb{R}^m$ be a manifold learning function, parameterized by an encoder.
This encoder maps data $x_i$ from the input space $\mathcal{X}$ to latent coordinates $z_i = \mathcal{G}(x_i) \in \mathcal{M}$ on the manifold. In $\mathcal{M}$, we calculate the probability by $q_{ij} = (1 + a || z_i - z_j||^{2b})^{-1}$. 
Instead of the Gaussian distribution, the fatter-tailed Student’s t-distribution is used to overcome the ``crowding problem" \cite{van2008visualizing} by allowing distant points in the high-dimensional space to be modeled as farther apart in the low-dimensional representation.
After obtaining the distribution $P$ in input space $\mathcal{X}$ and the distribution $Q$ in the manifold space $\mathcal{M}$, we can calculate the cross-entropy \cite{sainburg2021parametric} between them:
\begin{align}
    L_1 = \sum_{i\neq j} p_{ij} log(\frac{p_{ij}}{q_{ij}}) + (1-p_{ij}) log(\frac{1 - p_{ij}}{1 - q_{ij}})
\end{align} 

We now introduce the Manifold-Constrained NODE framework, which integrates manifold learning with NODEs to model dynamics that respect these geometric constraints.

\begin{definition}[Manifold-Constrained NODE]
A Manifold-Constrained NODE is a neural architecture comprising:
An encoder \(\mathcal{G}: \mathbb{R}^n \rightarrow \mathbb{R}^m\) that parameterizes a smooth manifold \(\mathcal{M}\) by mapping input data to a latent space. A Neural ODE that learns a smooth vector field \(f: \mathcal{M} \times \mathbb{R} \rightarrow T\mathcal{M}\) (where \(T\mathcal{M}\) is the tangent bundle of \(\mathcal{M}\)) to model the dynamics \(\frac{dh_{\mathcal{G}}}{d\tau} = f(h_{\mathcal{G}}, \tau, \theta)\) on \(\mathcal{M}\).
\end{definition}
The evolution of latent features $\frac{dh_\mathcal{G}}{dt}$ is guided by a vector field, which can be modeled by a neural network $f$. This is expressed as:
\begin{align}
    h_\mathcal{G}(T) = 
    \mathcal{G}(x_0) + \int_0^T f(h_\mathcal{G}, \tau, \theta) d\tau \in \mathcal{M}.
\end{align}
We give a theorem of the existence of dynamics on manifolds and the proof.
\begin{theorem}[Existence of Dynamics on Manifolds]
Let \(\mathcal{M}\) be a smooth, differentiable manifold embedded in \(\mathbb{R}^m\), and let \(f: \mathcal{M} \times \mathbb{R} \rightarrow T\mathcal{M}\) be a smooth, Lipschitz continuous vector field. For any initial condition \(h_{\mathcal{G}}(0) = \mathcal{G}(x_0) \in \mathcal{M}\), there exists a unique solution to the initial value problem:
\[
\frac{dh_{\mathcal{G}}}{d\tau} = f(h_{\mathcal{G}}, \tau, \theta), \quad h_{\mathcal{G}}(0) = \mathcal{G}(x_0),
\]
with \(h_{\mathcal{G}}(\tau) \in \mathcal{M}\) for all $\tau$.
\label{Existence_of_Dynamics_on_Manifolds}
\end{theorem}

\begin{proof} 
Since \(\mathcal{M}\) is a smooth manifold, it admits local coordinate charts that are diffeomorphic to subsets of \(\mathbb{R}^m\). In these coordinates, the vector field \(f\) can be expressed as a system of ODEs in Euclidean space. The Lipschitz continuity of \(f\) ensures the applicability of the Picard-Lindel\"of theorem \cite{teschl2012ordinary}, guaranteeing a unique solution locally. By the manifold's smoothness and the encoder’s parameterization, these local solutions can be patched together globally, keeping trajectories on \(\mathcal{M}\).
\end{proof}
In downstream tasks, such as classification, the function $f$ learns the trajectories of dynamics from the latent feature space to a label space.
Another cross-entropy loss, combined with softmax, is employed for this purpose:
$L_2 = -\sum_i^{C} y_i log (c_i)$. Here, $C$ represents the number of classes, $c_i$ is the predicted probability of an instance belonging to class $i$ and $y_i$ is the true label.
Therefore, the loss of our model is a combination of these two components:
\begin{equation}
-\sum_i^{C} y_i log (c_i) + \sum_{i\neq j} p_{ij} log(\frac{p_{ij}}{q_{ij}}) + (1-p_{ij}) log(\frac{1 - p_{ij}}{1 - q_{ij}})
\end{equation}
The encoder $\mathcal{G}$ for the manifold and the neural network $f$ for the dynamic learning will simultaneously be optimized by gradient descent during the training process. 
Joint optimization ensures the manifold $\mathcal{M}$ and dynamics co-evolve to satisfy both geometric fidelity and task performance.
The pseudo-code of Manifold-Constrained NODE for the classification task is shown in Algorithm \ref{alg:manifold_constrained_node}.

\begin{algorithm}[ht]
\caption{Manifold-Constrained NODE}
\label{alg:manifold_constrained_node}
\begin{algorithmic}[1]
\Require 
\Statex $X = \{x_i\}_{i=1}^N$: dataset
\Statex $k$: \# of nearest neighbors
\Statex $a, b$: Student's t-dist. parameters
\Statex $\eta$: learning rate
\Statex $\theta_\mathcal{G}, \theta_f$: parameters in encoder and NODE function 
\For{epoch $= 1$ to max\_epochs}
    \For{each $x_i \in X$}
        \State Find $k$ nearest neighbors and compute distances $d(x_i,x_j)$
        \State $\rho_i \gets \min_j d(x_i,x_j)$
        \State $p_{j|i} \gets \exp\bigl[-\bigl(d(x_i,x_j) - \rho_i\bigr)/\sigma_i\bigr]$
    \EndFor
    \State $p_{ij} \gets p_{j|i} + p_{i|j} - p_{j|i}\,p_{i|j}\quad \forall i \neq j$
    \State $z_i \gets \mathcal{G}(x_i; \theta_\mathcal{G})\quad \forall i$
    \State $q_{ij} \gets \bigl(1 + a\|z_i - z_j\|^{2b}\bigr)^{-1}\quad \forall i \neq j$
    \State $Loss_1 \gets \displaystyle
        \sum_{i \neq j} \biggl[
            p_{ij} \ln\!\bigl(\tfrac{p_{ij}}{q_{ij}}\bigr)
          + (1-p_{ij}) \ln\!\bigl(\tfrac{1 - p_{ij}}{1 - q_{ij}}\bigr)
        \biggr]$
    \For{each $x_i$}
        \State $h_\mathcal{G}^i(T) = z_i + \int_0^T f(h_\mathcal{G}, \tau, \theta) d\tau \in \mathcal{M} $ 
        \State $c_i \gets \text{softmax}\bigl(\text{classifier}(h_\mathcal{G}^i(T))\bigr)$
    \EndFor
    
    \State $Loss_2 \gets -\sum_{i=1}^N \sum_{c=1}^C y_{i,c}\,\ln\bigl(c_{i,c}\bigr)$
    \State $Loss \gets Loss_1 + Loss_2$
    
    \State $\theta_\mathcal{G}, \theta_f \gets \theta_\mathcal{G}, \theta_f \;-\; \eta \;\nabla_{\theta_\mathcal{G},\theta_f} L$
\EndFor

\State \textbf{return} $\theta_\mathcal{G}, \theta_f$
\end{algorithmic}
\end{algorithm}

\section{Numerical Experiments}
\label{Experiments}
We will demonstrate the superiority of our methodology in terms of accuracy, NFEs, and convergence speed. 
In Section \ref{Experimental Setup}, we introduce the datasets and environment settings. 
In Section \ref{Learning Dynamics on known Spherical Space Manifold}, we present a special case where the manifold is known, specifically focusing on a spherical space manifold. Its purpose is to illustrate the critical role of manifold structure in enhancing the efficacy of dynamic learning.
In Section \ref{Image Classification Results} we show on three real-life image datasets that our model has better prediction accuracy, fewer NFEs, and faster convergence speed compared to baselines.
In Section \ref{Series Classification Results}, we apply our approach to three series datasets.
All the models were implemented in Python 3.9 and realized in PyTorch.
We employed a high-performance computing server equipped with NVIDIA A100-SXM4-80GB GPUs to train and evaluate all models and perform additional analysis.

\subsection{Experimental Setup}
\label{Experimental Setup}
\noindent \textbf{Datasets.} 
We evaluated our model with three image classification datasets and three series classification datasets.
For the image classification task, we evaluate our model on the MNIST ~\cite{deng2012mnist}, CIFAR-10 ~\cite{krizhevsky2009learning}, and SVHN ~\cite{netzer2011reading}.
MNIST is a handwritten digit database with a training set of $60,000$ samples. The CIFAR-10 training dataset consists of $60,000$ $32 \times 32$ color images in ten classes. SVHN is a digit classification dataset that contains $600,000$ $32\times32$ RGB images of printed digits (from $0$ to $9$) cropped from pictures of house number plates.
For series datasets, we use BeetleFly, HandOutlines, and ECG200, which come from ~\cite{bagnall2018uea}.
BeetleFly is a dataset that distinguishes between beetles and flies, where the outline of the original image is mapped to a one-dimensional series at a distance from the center.
HandOutlines is designed to test the efficacy of hand and bone outline detection and whether these outlines could be helpful in bone age prediction. 
ECG200 is a binary classification dataset that traces the electrical activity recorded during one heartbeat. The two classes are a normal heartbeat versus a myocardial infarction event.

\noindent \textbf{Evaluation metrics and baselines.} 
For the image and series classification task, we compared our model with NODEs, ANODEs\cite{dupont2019augmented}, CNN, Res-Net with $10$ residual blocks implemented in ~\cite{lin2018resnet}, and PCA+ANODEs in terms of test accuracy.
We also compared ours with ODE-based models in terms of NFEs, and convergence speed. 

\noindent \textbf{Parameter settings.} For image datasets, we set the batch size as $32$. We use the same vector field modeling in all baseline continuous models. The vector field is modeled by three convolutional layers. The in channels and out channels for each layer are set as $(N_{in}, 32)$, $(32, 32)$, and $(32, N_{in})$ respectively, where the $N_{in}$ represents the number of channels of the input image. The setting of CNN is the same as the convolutional vector field. For our model, we use the three-layer MLP to model the vector field since the input of NODE is flattened. The middle layer has $64$ neurons. 
For the encoder, we use two convolutional layers and three fully connected layers. The in channels and out channels for each convolutional layer are set as $(N_{in}, 64)$ and $(64,128)$. The kernel size and stride are set as $3$ and $2$ respectively. 
We use the $ReLU$ as the activation function. 
For the building process of graph structure, the number of neighbors is set as $15$ as the hyperparameter, which in our experience is not a sensitive one.
We use the Adam algorithm as the optimizer with a learning rate of $10^{-3}$. 
We run five epochs for each experiment since the experiment shows that five epochs are enough to converge. 
For Res-Net, we model it using $10$ residual blocks, and each block is implemented by a two-layer MLP ~\cite{lin2018resnet}.
For series classification tasks, we use the same vector field modeling in all continuous models. The vector field is implemented by a three-layer MLP with the hidden dimensions as $16$. 
We run $30$ epochs for each experiment. 
For all the continuous models, we set the same tolerance of the ODE solver, as $10^{-3}$. 
For all the augmented models, we use five extra dimensions. 

\subsection{Learning Dynamics on Spherical Manifold}
\label{Learning Dynamics on known Spherical Space Manifold}
To show the critical role of manifold structure in enhancing the efficacy of dynamic learning. We first learn the dynamics in a known manifold space.
Consider a specific scenario where the dynamics unfold within a spherical space with a radius of $R = 1$, referred to as $\mathcal{S}$.
In this context, it is known that the solution evolves within a submanifold of $\mathbb{R}^3$, and the vector field $f$ is defined on this submanifold.
Let $\mathcal{G}$ represent a manifold learning function defined as follows: $\mathcal{G}: \mathbb{R}^3 \rightarrow \mathcal{S} \subset \mathbb{R}^3$. In simpler terms, $\mathcal{G}$ is a function that maps from three-dimensional Euclidean space to a submanifold $\mathcal{S}$ embedded within three-dimensional Euclidean space.
Define $h$ as the state in three-dimensional Euclidean space, represented as 
$h = \left[ \begin{array}{c} x\\ y\\ z \end{array} \right] \in \mathbb{R}^3$. 
On the other hand, $l$ is the state within the submanifold, expressed as 
$l= \left[ \begin{array}{c} u\\v \end{array} \right] \in \mathbb{R}^2$. 
To establish a connection between the two representations, we can relate $u$ and $v$ to $h$ using the following equations:
$h = \left[ \begin{array}{c}
      x\\ 
      y\\ 
      z   
      \end{array} \right]
   = \left[ \begin{array}{c}
      R\cdot sin(u) cos(v)\\ 
      R \cdot sin(u) sin(v)\\ 
      R \cdot cos(u)  
      \end{array} \right].$
The derivative of $h$ with respect to $t$ represents the rate of change of state $h$ with respect to time:
    $
    \frac{dh}{dt} 
    = 
    \frac{dh}{dl} \cdot \frac{dl}{dt} 
    = 
    \left[ \begin{array}{c c}
    \frac{\partial x}{\partial u} & \frac{\partial x}{\partial v}\\ 
    \frac{\partial y}{\partial u} & \frac{\partial y}{\partial v}\\ 
    \frac{\partial z}{\partial u} & \frac{\partial z}{\partial v}
    \end{array} \right] \cdot \frac{dl}{dt} 
    =
    R \cdot
    \left[\begin{array}{cc}
    cos(u)cos(v) & - sin(u)sin(v)\\ 
    cos(u)sin(v) &  sin(u)cos(v) \\ 
    -sin(u)  & 0
    \end{array} \right]
    \cdot
    \left[\begin{array}{c}
    \frac{du}{dt}\\ 
    \frac{dv}{dt}
    \end{array} \right].
    $
Considering $\frac{dl}{dt}$ as the vector field within the manifold, we employ a neural network denoted as $f: l \rightarrow \frac{dl}{dt}$ to model this vector field. Function $f$ describes the evolution of the state $l$ within the manifold.
Given an initial state $h(0)$ in the original space, we integrate $\frac{dh}{dt}$ over time to derive the final state $h(T)$:
    $
    h(T) 
    = 
    h(0) + \int_0^T \frac{dh}{dt} dt
    =
    h(0) + \int_0^T R \cdot
    \left[\begin{array}{cc}
    cos(u)cos(v) & - sin(u)sin(v)\\ 
    cos(u)sin(v) &  sin(u)cos(v) \\ 
    -sin(u)  & 0 \end{array} \right] 
    \cdot f(l, t, \theta) dt.
  $

\begin{figure}[t]
\begin{center}
\includegraphics[trim=0cm 0.5cm 0cm 0cm, clip, width=0.9\columnwidth]{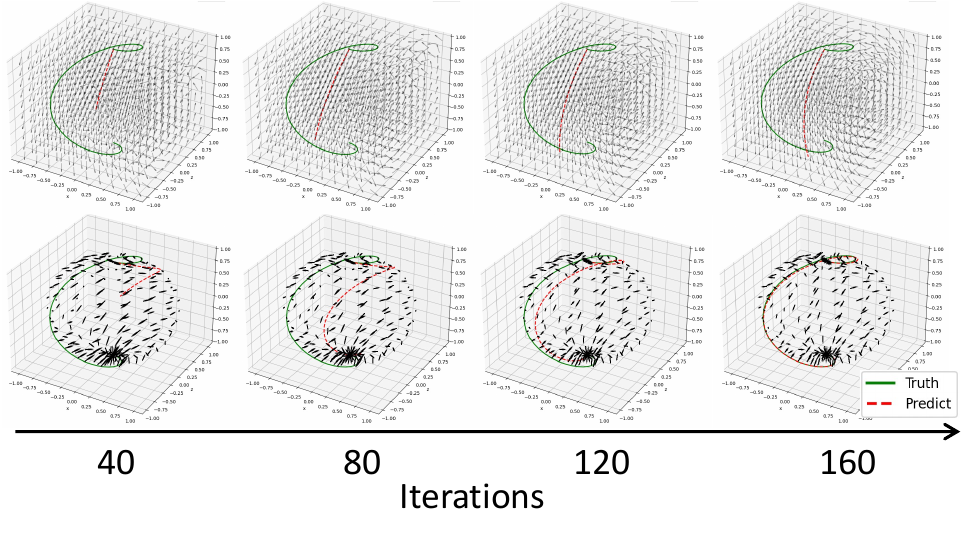}
\end{center}
\caption{Top: Vector fields from a NODE in three-dimensional Euclidean space. Bottom: Vector fields from a Manifold-constrained NODE in spherical space. 
}
\label{fig: toy_vector_field}
\end{figure}

\begin{figure}[t]
\begin{center}
\includegraphics[trim=0cm 0.5cm 0cm 0cm, clip, width=0.9\columnwidth]{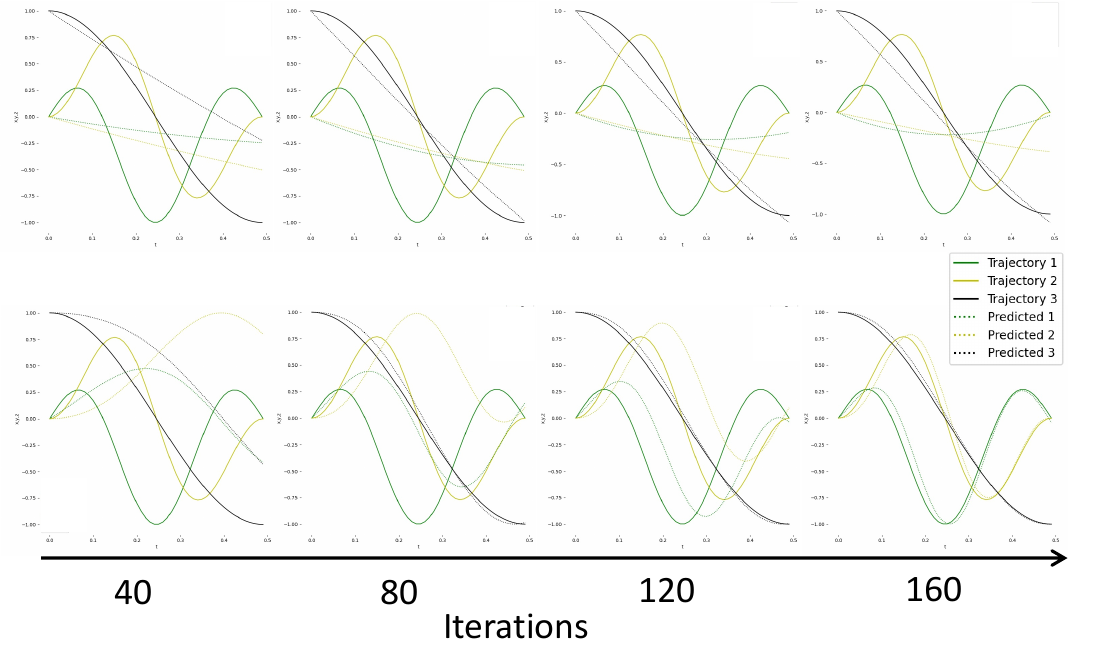}
\end{center}
\caption{Top: Trajectories in three-dimensional Euclidean space learned by NODE. Bottom: Trajectories learned by Manifold-constrained NODEs. Solid lines depict true trajectories, and dotted lines indicate the learned trajectories.}
\label{fig: toy_trajectory}
\end{figure}

The latent state $l$ within the manifold offers a more robust and expressive representation compared to the latent state $h$ within the original space. The evolution of vector fields is shown in Figure \ref{fig: toy_vector_field}. The vector field learned by the NODEs in manifold makes dynamic more explainable and easier to converge during learning.
The trajectories of these dynamics are visualized in Figure \ref{fig: toy_trajectory}. It also demonstrates the advantages of learning dynamics on the manifold.

\subsection{Image Classification with Manifold-Constrained NODE}
\label{Image Classification Results}
Considering images are usually in high dimensions, we apply our method to image classification tasks. MNIST dataset inherently inhabits a 784-dimensional space ($1 \times 28 \times 28$), and datasets like CIFAR-10 or SVHN, originally occupy a 3072-dimensional space ($3 \times 32 \times 32$). 
In image classification, the dynamics often lie in the transformation from the original or latent space to the output space, and NODEs provide a novel approach to modeling these transformations continuously.
Unlike traditional discrete methods, NODEs offer a framework for understanding the continuous trajectories of data through the model, which can enhance interpretability and act as a form of regularization. This continuous approach is particularly insightful for analyzing how input features evolve into outputs, making the learning process more transparent and interpretable.

\noindent \textbf{Graph Representation.}
In our experiment, we constructed the graph representation of the provided dataset by determining the global probability in Equation \ref{global_probability}. Each pair of data points are connected by a weighted edge. These weights are based on the global probability. We demonstrate this with an example using the MNIST dataset. For visualization purposes, we only display edges with a global probability exceeding $0.5$. As depicted in the left panel of Figure \ref{fig: graph}, it is evident that samples sharing the same label are more likely to be connected, indicating a higher probability of linkage. On the right side of Figure \ref{fig: graph}, we present the weight matrix, which has been organized according to the labels of the samples. This arrangement further highlights that samples with identical labels tend to have a higher global probability of connection.

\begin{figure}[t]
\begin{center}
\includegraphics[trim=0cm 0.2cm 0cm 0cm, clip, width=0.9\columnwidth]{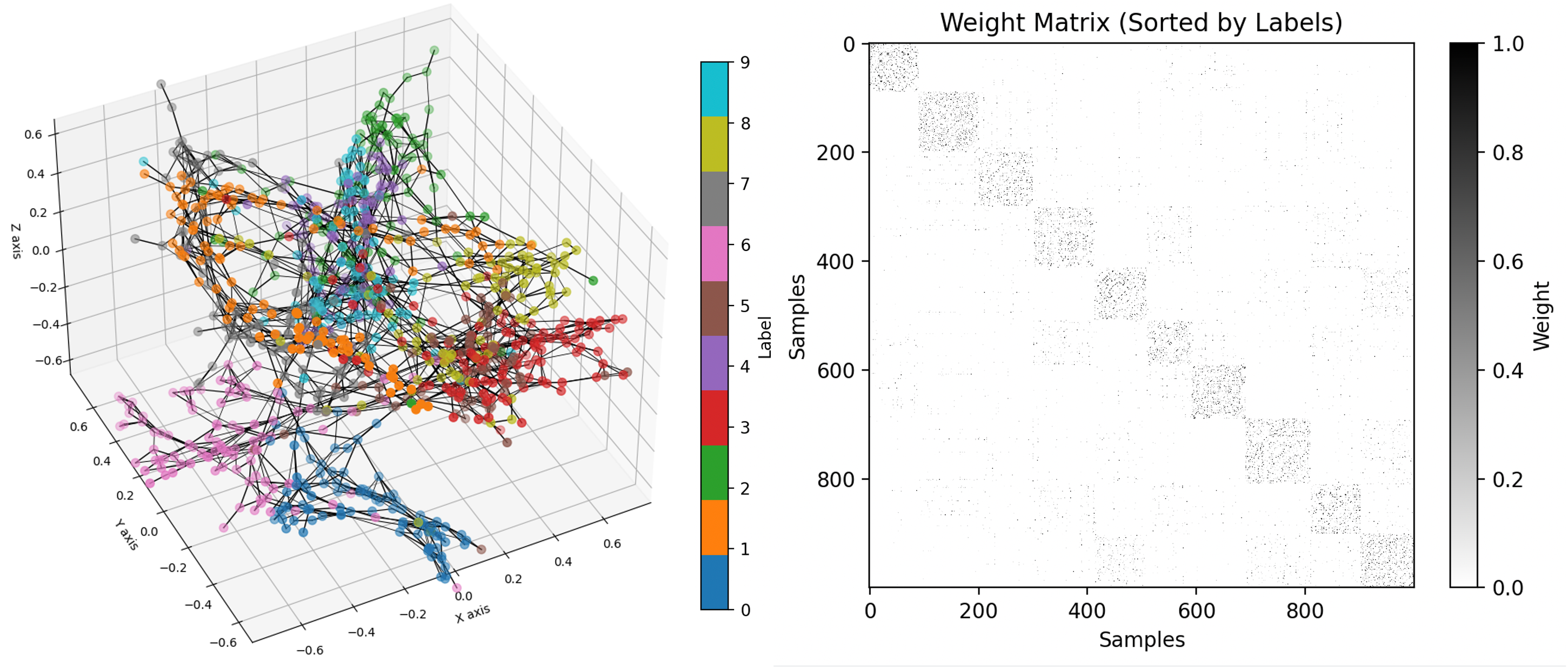}
\end{center}
\caption{The left plot illustrates the graph structure constructed for the MNIST dataset, based on global probability. In this visualization, only edges with a probability exceeding $0.5$ are displayed. The right plot presents the weighted matrix between samples, with the samples organized according to their labels.}
\label{fig: graph}
\end{figure}

Once we obtain the structure of the dataset, we employ an encoder to initialize the projection from the original space to the manifold space by minimizing the cross-entropy.
The NODE learns the dynamics in a low-dimension space where the manifold exists. During training, the encoder and the NODE are optimized synchronously.

\begin{table}
  \caption{Testing Accuracy on Image Classification}
  \label{Tab_test_acc_img}
  \centering
  \small
  \resizebox{0.6\columnwidth}{!}{
  \begin{tabular}{ccccccc}
    \toprule
             & CIFAR-10 & MNIST & SVHN  \\
    \midrule
    ResNet-10       & $0.607$ 
                    & $0.978$ 
                    & $0.604$ 
                    \\
    CNN             & $0.636$ 
                    & $0.978$ 
                    & $0.834$ 
                    \\  
    NODE            & $0.602$ 
                    & $0.944$ 
                    & $0.758$ 
                    \\
    ANODE         & $0.618$ 
                    & $0.981$ 
                    & $0.568$ 
                \\ 
    PCA+ANODE & $0.237$ 
                  & $0.385$ 
                 & $0.302$
                 \\
    \midrule
    Manifold NODE        & $\textbf{0.672}$ 
                & $\textbf{0.985}$ 
                & $\textbf{0.840}$
                \\
    \bottomrule
  \end{tabular}
  }
\end{table}

\begin{figure}[htb]
\centering 
\includegraphics[trim=0cm 0.3cm 0cm 0cm, clip, width=1\columnwidth]{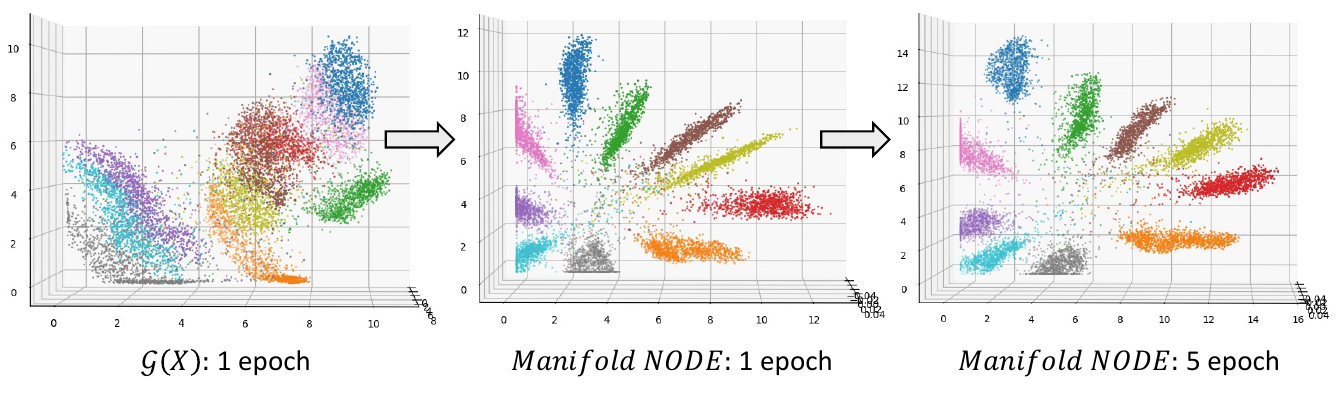}
\caption{The figure shows the evolution of MNIST test samples in three-dimensional space: after one epoch of encoder training (left), one epoch of full model training (middle), and five epochs of full model training (right).}
\label{fig: changing_during_training}
\end{figure}

\noindent \textbf{Test Accuracy.}
In our evaluation, we benchmark our approach against a range of baseline models, encompassing both continuous models like NODEs and ANODEs, as well as discrete models such as ResNet and CNNs. To validate the effectiveness of our method, we include a comparison with a hybrid approach that combines general dimensionality reduction techniques, PCA, with NODEs. For a fair comparison, we align the number of principal components ($100$) in PCA with the dimensions of the manifold space used in our method.

However, it is noteworthy that the performance of the PCA combined with ANODEs is markedly poor. This outcome underscores a critical insight: merely applying simplistic dimensionality reduction techniques, without a nuanced exploration of the data's manifold structure, is insufficient for achieving optimal results. This finding highlights the importance of more sophisticated approaches, such as the one we propose, in effectively capturing and utilizing the complex underlying structures in the data.

We present a visualization of the evolution of MNIST test samples in a three-dimensional space. 
In Figure \ref{fig: changing_during_training}, the first plot captures the positions after training only the encoder for one epoch. The second plot depicts the sample positions after one epoch of training the entire model. The third plot shows the positions after completing five epochs of training with our model. This visual progression demonstrates the dynamic changes in the data representation as the model training progresses. It highlights the increasing separation of samples with distinct labels for the training period.


\subsection{Series Classification with Manifold-Constrained NODE}
We further evaluate our method on three real-world time series datasets (BeetleFly, HandOut, and ECG200), comparing it with both discrete (ResNet-10, CNN) and ODE-based (NODE, ANODE) baselines. Table~\ref{test_acc_and_avg_NFEs_time_series} shows that our Manifold NODE consistently achieves the highest accuracy, confirming its effectiveness and robustness for continuous-time classification on lower-dimensional manifolds.

\label{Series Classification Results}
\begin{table}
  \caption{Testing Accuracy on Series Classification}
  \label{test_acc_and_avg_NFEs_time_series}
  \centering
  \small
  \resizebox{0.6\columnwidth}{!}{
  \begin{tabular}{ccccccc}
    \toprule
    & BeetleFly  & HandOut & ECG200\\
    \midrule
    ResNet-10       & $0.787$ 
                    & $0.813$ 
                    & $0.801$ 
                    \\
    CNN             & $0.789$ 
                    & $0.891$ 
                    & $0.811$ 
                    \\  
    NODE            & $0.800$ 
                    & $0.897$ 
                    & $0.836$  
                    \\
    ANODE           & $0.816$ 
                    & $0.889$ 
                    & $0.836$ 
                    \\
    PCA+ANODE       & $0.613$ 
                    & $0.735$ 
                    & $0.704$ 
                    \\
    \midrule
    Manifold NODE 
                    & $\textbf{0.867} $ 
                    & $\textbf{0.915} $ 
                    & $\textbf{0.862} $
                    \\
    \bottomrule
  \end{tabular}
  }
\end{table}

\subsection{Dimensionality Sensitivity Analysis}
We performed a dimension sensitivity analysis on three image datasets to evaluate the impact of varying dimensionality on performance in different potential manifold spaces. As shown in Figure \ref{fig: DimensionSensitivity}, our model consistently delivers high accuracy, even with the dimensionality reduced to $20$. Notably, our method sustains strong accuracy on MNIST and CIFAR-10, even when the dimensionality is decreased to three. This result emphasizes our model's capability to retain the intrinsic structure of data across a broad spectrum of manifold spaces with varying dimensions. Such adaptability to diverse dimensional spaces underscores the model's versatility and robustness, particularly in processing complex image datasets.

\begin{figure}[h]
\begin{center}
\includegraphics[width=0.9\columnwidth]{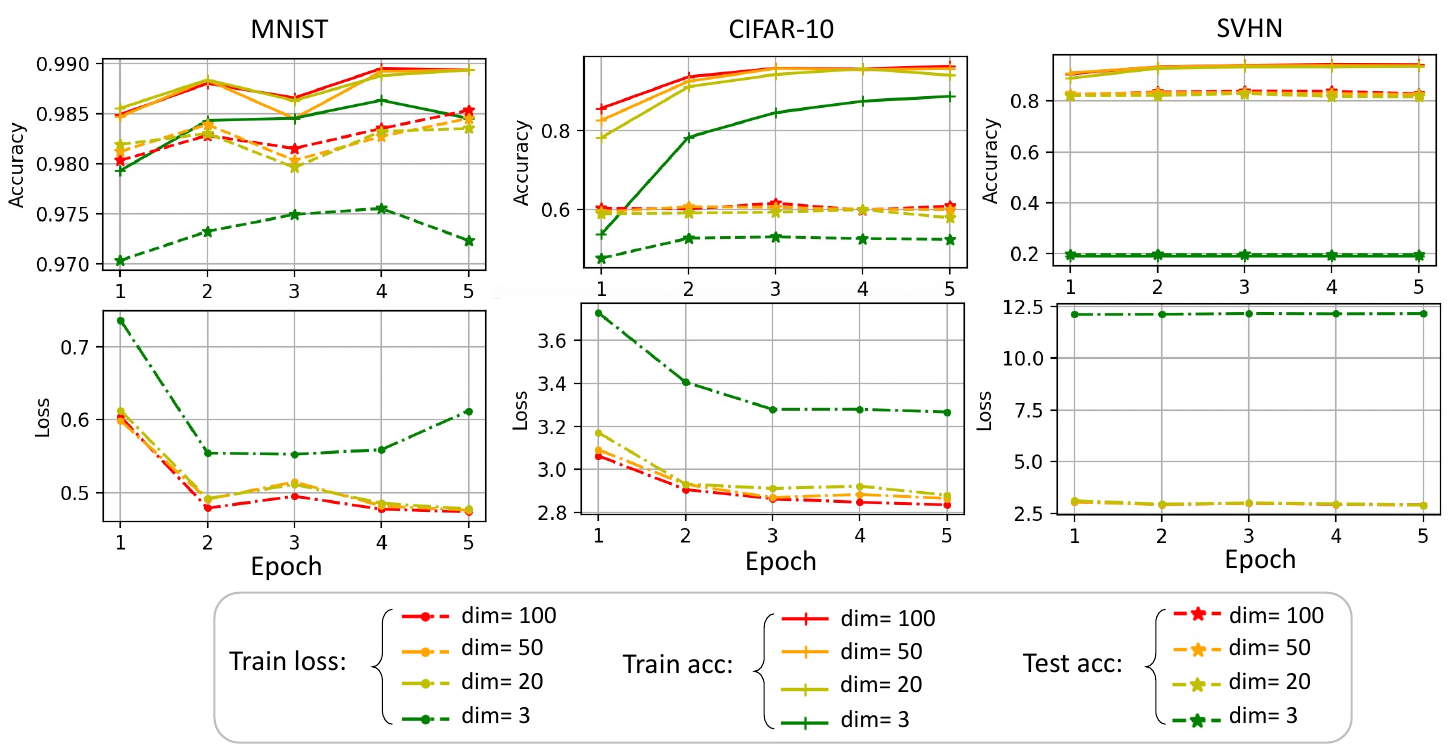}
\end{center}
\caption{Dimension sensitivity analysis across three datasets (MNIST, CIFAR-10, SVHN), shown from left to right. Top: Training and test accuracies. Bottom: Training loss.}
\label{fig: DimensionSensitivity}
\end{figure}

\subsection{NFEs and Time}
The NFEs play a vital role in defining the computational cost and efficiency of an ODE-based model. This metric indicates how many times the ODE solver computes the ODE's function while solving the ODE. In a consistent environment, an ODE-based model that necessitates fewer function evaluations typically demonstrates greater efficiency in training. Our approach, which focuses on learning a simpler vector field within a manifold, is designed to be more efficient in this regard. As a result, it requires fewer function evaluations compared to traditional NODEs (NODEs) and Augmented NODEs (ANODEs). This efficiency potentially leads to faster training times and reduced computational resource usage, while still maintaining the model's effectiveness.

\begin{figure}[h]
\centering 
\includegraphics[width=0.9\columnwidth]{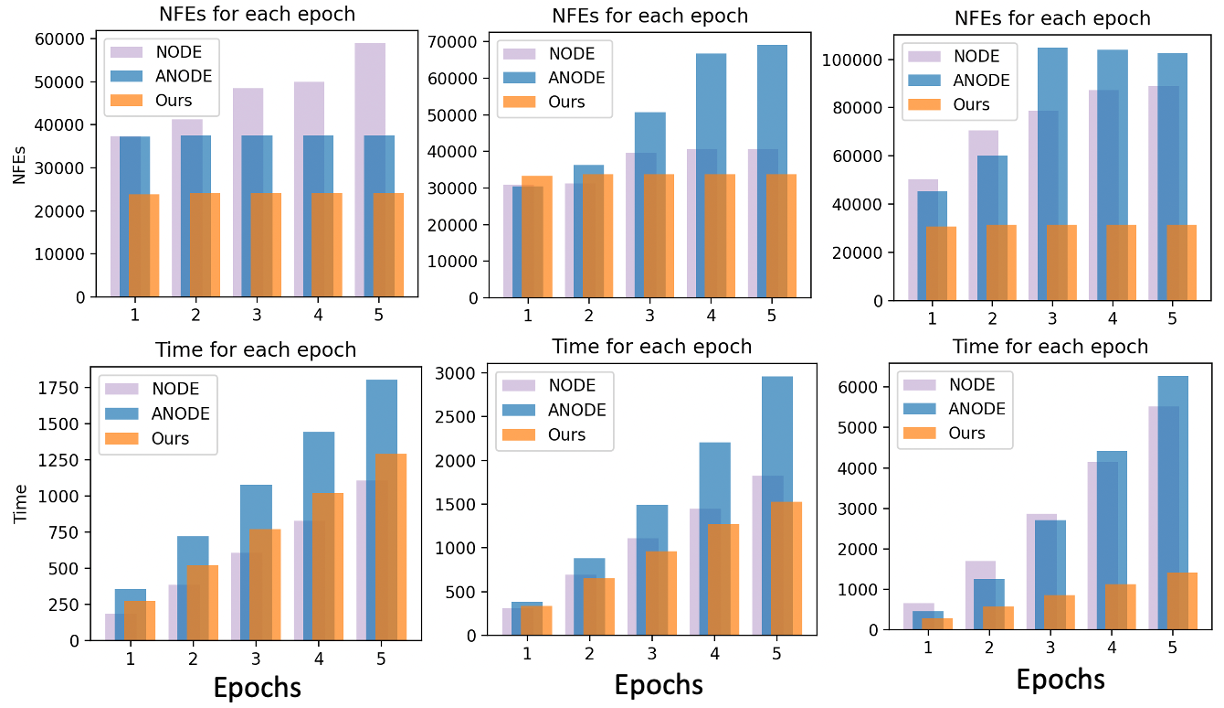}
\caption{The top three plots show the sum of NFEs for each epoch. The bottom three plots show the time of each epoch. From left to right, it is MNIST, CIFAR-10, and SVHN, respectively.}
\label{fig: NFEs_and_TIME}
\end{figure}

To test this, we measure the sum of NFEs in each epoch during the training process.
We visualize the NFEs in Figure \ref{fig: NFEs_and_TIME}. 
Our model's NFEs are fewer than baseline models. Meanwhile, we notice that our models' NFEs can maintain a relatively stable level whereas the baseline models' NFEs will increase rapidly with the training epoch increase. This is one of the reasons that our model has a fast convergence rate.

\section{Conclusion}
In this work, we introduced a novel approach to address the challenges of learning dynamics in high-dimensional space. 
By integrating manifold learning principles with NODEs, our method offers an efficient and accurate solution for dynamic learning.
We leverage the manifold hypothesis and project the original data into the manifold by an encoder while preserving the data structure.
Our methodology allows us to reduce complexity while preserving accuracy in dynamic learning. 
Experimental evaluations across diverse datasets consistently demonstrated our approach's superiority, underscoring its potential to advance our understanding of high-dimensional systems and improve modeling accuracy.

\bibliographystyle{IEEEtran} 
\bibliography{IEEEexample}  

\end{document}